\documentclass[conference]{IEEEtran}
\usepackage{cite}
\ifCLASSINFOpdf
\usepackage[pdftex]{graphicx}
\DeclareGraphicsExtensions{.pdf,.jpeg,.png}
\else
\usepackage[dvips]{graphicx}
\DeclareGraphicsExtensions{.eps}
\fi
\usepackage{epsfig,epsf,epstopdf,graphicx}
\usepackage[cmex10]{amsmath}
\usepackage{amssymb}
\usepackage{bbm}
\usepackage{amsthm }
\usepackage{nicefrac}
\usepackage{hyperref}
\usepackage{xcolor}
\hypersetup{colorlinks=true}
\usepackage{tabularx}

\newtheorem{lemma}{Lemma}

\theoremstyle{theorem}

\theoremstyle{definition}

\theoremstyle{plain}

\theoremstyle{plain}

\newcommand{\muv}{{\mbox{\boldmath $\mu$}}}

\newcommand{\av}{{\bf{a}}}
\newcommand{\bv}{{\bf{b}}}

\newcommand{\mv}{{\bf{m}}}

\newcommand{\xv}{{\bf{x}}}

\newcommand{\zerov}{{\bf{0}}}
\newcommand{\onev}{{\bf{1}}}

\newcommand{\Ab}{\mathbf{A}}
\newcommand{\Bb}{\mathbf{B}}

\newcommand{\Gb}{\mathbf{G}}
\newcommand{\Hb}{\mathbf{H}}
\newcommand{\Ib}{\mathbf{I}}

\newcommand{\Lb}{\mathbf{L}}
\newcommand{\Mb}{\mathbf{M}}

\newcommand{\Pb}{\mathbf{P}}

\newcommand{\Rb}{\mathbf{R}}
\newcommand{\Sb}{\mathbf{S}}

\newcommand{\Ub}{\mathbf{U}}

\newcommand{\Xb}{\mathbf{X}}
\newcommand{\Yb}{\mathbf{Y}}

\newcommand{\Sigmab}{\mathbf{\Sigma}}

\newcommand{\Real}{\mathbb{R}}

\newcommand{\Id}{\mathbbm{1}}

\newcommand{\Ic}{{\cal{I}}}

\newcommand{\Nc}{{\cal{N}}}
\newcommand{\Sc}{{\cal{S}}}

\newcommand{\Pc}{{\cal P}}

\newcommand\normop[1]{\left\lVert#1\right\rVert}

\newcommand{\argminn}{\mathop{ \mathrm{argmin}}}

\newcommand{\tr}{\mathop{ \mathrm{tr}}}

\newcommand{\diag}{\mathrm{diag}}
\newcommand{\Diag}{\mathrm{Diag}}

\newcommand{\rank}{\mathrm{rank}}



\usepackage{algorithm}
\usepackage{multirow}
\usepackage{algcompatible}
\usepackage[bottom]{footmisc}
\usepackage{algpseudocode}
\usepackage{enumitem} 
\PassOptionsToPackage{dvipsnames}{xcolor}
\usepackage{tikz}
\usetikzlibrary{calc}

\usepackage{subcaption}
\usepackage{fixltx2e}
\usepackage{afterpage}
\usepackage{float}
\usepackage{stfloats}

\usepackage[space]{grffile}

\makeatletter

\begin{document}

\title{Clustering of Incomplete Data via a Bipartite Graph Structure
\vspace*{-0.2cm}}

\author{\IEEEauthorblockN{Amirhossein Javaheri and~Daniel~P.~Palomar}
\IEEEauthorblockA{
\textit{Hong Kong University of Science and Technology}\\
sajavaheri@connect.ust.hk, palomar@ust.hk\vspace*{-0.4cm}}
\thanks{
% This work was supported by the Hong Kong GRF 16206123 research grant.
} 
\vspace{-0.5cm}}

\date{}
\maketitle
\thispagestyle{plain}
\pagestyle{plain}

\begin{abstract}
% Graph learning 
% is an emerging research area 
% with many applications in machine learning, specifically data clustering.
There are various approaches to graph learning for data clustering, incorporating different spectral and structural constraints through diverse graph structures. Some methods rely on bipartite graph models, where nodes are divided into two classes: centers and members. These models typically require access to data for the center nodes in addition to observations from the member nodes. However, such additional data may not always be available in many practical scenarios.  
Moreover, popular Gaussian models for graph learning have demonstrated limited effectiveness in modeling data with heavy-tailed distributions, which are common in financial markets.  
In this paper, we propose a clustering method based on a bipartite graph model that addresses these challenges. First, it can infer clusters from incomplete data without requiring information about the center nodes. Second, it is designed to effectively handle heavy-tailed data. Numerical experiments using real financial data validate the efficiency of the proposed method for data clustering. 

\end{abstract}


% \fi
\IEEEpeerreviewmaketitle

\section{Introduction}
\label{sec:Intro}

\IEEEPARstart{C}{lustering}  is a fundamental technique in data mining and machine learning \cite{
% jain_data_1999, 
% oyelade_data_2019, 
oyewole_data_2023}, with applications spanning 
bio-informatics \cite{eisen_cluster_1998}
% biomedical sciences \cite{hu_spectral_2015},
image segmentation  \cite{shi_normalized_2000}, social networks  \cite{singh_clustering_2016}, and financial markets \cite{marti_review_2021}.
% dose_clustering_2005,
% lemieux_clustering_2014
% }.
Among various clustering techniques including
hierarchical \cite{johnson_hierarchical_1967}, density-based \cite{martin_density_1996} and model-based methods \cite{mclachlan_finite_2004}, there are approaches
that leverage graphical models to capture similarities between data points \cite{paratte_graph-based_2017}.
Spectral graph clustering \cite{ng_spectral_2001} is an early example of such methods, utilizing heuristic similarity graphs to represent data relationships. A more advanced approach involves learning graphs from data 
under 
% using 
% Gaussian models, where a multivariate Gaussian distribution, specifically
a Gaussian Markov random field (GMRF) model.
% , models the data.
% In this context, the graph structure is inferred through maximum likelihood (ML) or maximum a posteriori (MAP) estimation. 
The Graphical LASSO (GLASSO) \cite{friedman_sparse_2008} was a foundational method in this area, later enhanced by introducing Laplacian structural constraints to the graph learning problem \cite{lake_discovering_2010, egilmez_graph_2017, zhao_optimization_2019, ying_nonconvex_2020, javaheri_learning_2024}.

A more advanced approach in graph-based clustering incorporates spectral constraints into graph learning to infer specific structures, such as $k$-component graphs, which explicitly represent clusters \cite{nie_constrained_2016, kumar_unified_2020}. However, these methods often struggle to accurately model data with heavy-tailed distributions, which are prevalent in real-world applications like financial markets. Consequently, recent efforts have focused on generalizing $k$-component graph learning methods to handle heavy-tailed distributions \cite{cardoso2021graphical, javaheri_graph_2023, javaheri_eusipco_2024}.  
Another category of methods for clustering introduces $k$-component bipartite graph structures, comprising two types of nodes: cluster centers and their members \cite{nie_learning_2017, de_miranda_cardoso_learning_2022}. While effective, these approaches typically require additional data for the center nodes to learn the graph. This dependency on center node data, however, poses challenges in applications where such data is unavailable.  

In this paper, we propose a graph learning method for data clustering based on a bipartite graph structure. Our method is designed to handle scenarios with incomplete information about the center nodes and is robust to heavy-tailed data distributions, making it suitable for a wide range of real-world applications.

\subsection{Notations}
% The notations used in  the paper are as follows: 
Vectors and  matrices are respectively denoted with bold lowercase and uppercase letters.
(e.g., $\xv$ and $\mathbf{X}$). 
The $i$-th element of a vector $\xv$ is denoted by $x_i$. Also $X_{i,j}$ denotes the $(i,j)$-th element of $\Xb$.  
The notation $\xv_{i_1:i_2}$ is defined as $\xv_{i_1:i_2} \triangleq [x_{i_1},\ldots, x_{i_2}]$.  We use $\|\mathbf{x}\|$ to denote the $\ell_2$ norm of a vector and $\|\mathbf{X}\|_F$ to denote the Frobenius norm of a matrix. Diagonal elements of  $\mathbf{X}$ are shown with $\mathrm{diag}(\mathbf{X})$, while $\mathrm{Diag}(\mathbf{x})$ is a diagonal matrix with  $\mathbf{x}$ on its main diagonal. The notation $\mathrm{det}^\ast$ also represents the generalized determinant.

\section{Problem Formulation}
\label{sec:problem}
Consider a weighted undirected graph with $p$ vertices, where each node represents an element of a signal $\xv \in \Real^p$, and the weights of the edges encode the relationships between the elements. Suppose we have $n$ measurements of $\xv$, represented as the columns of $\Xb = [\xv_1, \hdots, \xv_n] \in \Real^{p \times n}$.  
The problem we investigate in this paper is to learn an undirected $k$-component bipartite graph 
% that can capture the underlying structure of the data. In specific, we intend to use the learned graph 
for clustering the data into $k$ clusters. 
% being represented with the components of the graph.
In this graph structure, inspired by \cite{de_miranda_cardoso_learning_2022}, one class of nodes represents the centers of the clusters, while the other class represents the members (to be clustered). 
Assume we divide the nodes into $k$ centers and $r$ members, where $r + k = p$. Let $B_{i,j}\geq 0$ be the weight of the edge connecting the member $i \in \{1, \cdots, r\}$ to the center $j \in \{1, \cdots, k\}$. Then, 
% we may state 
the Laplacian matrix of the graph can be expressed as
% our bipartite graphical model as follows:
\begin{align}
\label{eq:bi-Laplacian}
    \Lb =  \left[\begin{smallmatrix}
\Diag(\Bb \onev_k) & -\Bb \\
-\Bb^\top & \Diag(\Bb^\top \onev_r)
\end{smallmatrix}\right],
\end{align}
where $\Bb \in \Real_+^{r \times k}$.
% represents the weights of the edges.
The weight $B_{i,j}$ models the probability of member node $i$ being within cluster $j$,
Hence, the sum of each row of $\Bb$ 
equals one.

Next, we consider a stochastic approach to learning such graph structure from data. In specific, we assume $\xv_i$s are drawn from a zero-mean multivariate 
Student-t distribution as
\begin{align*}
    p(\xv_i)  
    \propto 
    \mathrm{det^*}(\Lb)^{1/2} \left(1+\frac{\xv_i^\top \Lb \xv_i}{\nu}\right)^{-(\nu + p)/2},\quad \nu>2.
\end{align*}

Under the above model, the problem of maximum likelihood estimation of the
graph from data  can be formulated by
\begin{equation*}
\begin{array}{cl}
\underset{\Lb,\, \Bb}{\mathsf{min}} & \hspace{-5pt}\begin{array}{ll}
\frac{p+\nu}{n} \sum_{i=1}^n \log \left(1 + \frac{\xv_i^\top\Lb\xv_i}{\nu}\right)
- \log \det^\ast(\Lb ) 
\end{array}\\
\mathsf{s.~t.} & 
\hspace{-5pt}\begin{array}[t]{l}
\Lb =  \left[\begin{smallmatrix}
\Diag(\Bb \onev_k) & -\Bb \\
-\Bb^\top & \Diag(\Bb^\top \onev_r)
\end{smallmatrix}\right], \,\, \rank(\Lb) = p-k, \\
\Bb \geq 0,\,\, \Bb \onev_k = \onev_r.
\end{array}
\end{array}
\end{equation*}

Suppose we are only given the first $r$ rows of the data matrix, corresponding to the members, which we refer to as the incomplete data denoted by $\tilde{\Xb} \in \Real^{r \times n}$. We may consider each unavailable row of the data matrix, corresponding to the centers, to be   a weighted average of the rows of $\tilde{\Xb}$ (members). Hence, the augmented data matrix yields as $\Xb = \left[\tilde{\Xb}^\top \quad \tilde{\Xb}^\top \Ab\right]^\top\in \Real^{p\times n}$, where $\Ab \in \Real_+^{r \times k}$ denotes the weight matrix
with non-negative elements. The sum of these weights for each center node equals unity, and hence, 
we have $\Ab^\top \onev_r = \onev_k$. We also assume $\Ab$ and $\Bb$ share the same support (indicating the membership sets of the clusters).  

% \begin{proposition}
\label{prop:1}
Consider the $i$-th column of $\Xb$ as $\xv_i^\top = [\tilde{\xv}_i^\top \,\, \tilde{\xv}_i^\top \Ab]^\top$, where $\tilde{\xv}_i$ is the $i$-th column of $\tilde{\Xb}$.  
Also let 
$\tilde{\Sb}_i = \tilde{\xv}_i \tilde{\xv}_i^\top$. Then with simple calculations we get:
    \begin{align}
    \label{eq:prop1}
    \begin{split}
        \xv_i^\top\Lb \xv_i &= h_i + \tr(\Bb \Gb_i(\Ab)),\quad  h_i= \tilde{\xv}_i^\top\tilde{\xv}_i = \tr(\tilde{\Sb}_i),\\
    \Gb_i(\Ab) &= -2 \Ab^\top \tilde{\Sb}_i + \diag \left(\Ab^\top \tilde{\Sb}_i \Ab \right) \onev_r^\top.
        \end{split}
\end{align}

Using this equation,
the problem would be restated as
\begin{equation}
\label{eq:problem_orig}
\begin{array}{cl}
    \begin{array}[t]{c}
    {\mathsf{min}}\\
    \scriptstyle \Lb ,\,\, \Bb,
    \,\, \Ab
    \end{array} & 
\hspace{-8pt}\begin{array}[t]{ll}
\frac{p+\nu}{n} \sum_{i=1}^n \log \left(1 + \frac{h_i + \tr(\Bb \Gb_i(\Ab))}{\nu}\right)
\\
- \log \det^\ast(\Lb ) 
\end{array} \\
\mathsf{s.~t.} & 
\hspace{-8pt}\begin{array}[t]{l}
\Lb =  \left[\begin{smallmatrix}
\Ib_r & -\Bb \\
-\Bb^\top & \Diag(\Bb^\top \onev_r)
\end{smallmatrix}\right]\!, \,\,\rank(\Lb) = p-k, \\
 \Bb \onev_k = \onev_r,\,\,
\Ab^\top\onev_r  = \onev_k,
\\
\Ab,\Bb \geq \zerov, {\,\, \color{black} \Ab \odot \Id(\Bb = \zerov) = \zerov,}
\end{array}
\end{array}
\end{equation}
where we have used  $\Bb\onev_k = \onev_r$ and $\Diag(\onev_r) =\Ib_r$ in the first equality constraint.

\section{Proposed Method}
\label{sec:proposed_heavy}

By relaxing the constraint $\Lb =  \left[\begin{smallmatrix}
\Ib_r & -\Bb \\
-\Bb^\top & \Diag(\Bb^\top \onev_r)
\end{smallmatrix}\right]$, 
% and $\Zb=\Ab$,
the augmented Lagrangian function yields as %\eqref{eq:problem_orig} 
% is formulated as follows
\begin{align}
\label{eq:lagrangian}
\hspace*{-2pt}L_\rho(\Ab, \Bb, \Lb
)
% , \Zb) 
&=
\frac{p+\nu}{n} \sum_{i=1}^n \log \left(1 + \frac{h_i + \tr(\Bb \Gb_i(\Ab))}{\nu}\right)
\nonumber\\
& -\log \det{\!^\ast}(\Lb )  + \frac{\rho}{2} \normop{\Lb - \left[\begin{smallmatrix}
\Ib_r & -\Bb \\
-\Bb^\top & \Diag(\Bb^\top \onev_r)
\end{smallmatrix}\right]\! }_{F}^2  
\nonumber \\
& 
+ \langle  \Lb - \left[ \begin{smallmatrix}
\Ib_r & -\Bb \\
-\Bb^\top & \Diag(\Bb^\top \onev_r)
\end{smallmatrix}\right], \Yb \rangle \nonumber\\
\end{align}
Now,  we use  a variant of the ADMM method \cite{boyd_distributed_2010} to solve  problem \eqref{eq:problem_orig} by  alternating minimization of the   augmented Lagrangian function.
Here, we have
3
% 4 
update steps corresponding to the primal variables $\Lb$, $\Bb$,
and $\Ab$, 
and
one update step
for the dual 
variable $\Yb$.
These  are given as follows:

\subsection{$\Lb$ update step:} 
The update step for $\Lb$ is obtained by solving the following subproblem
\begin{align*}
    \Lb^{l+1} = \hspace{-12pt}\begin{array}[t]{c}
    {\mathsf{argmin}}\\
    \scriptstyle \Lb\succeq \zerov,\\[-2pt]
    \scriptstyle \rank(\Lb)\,=\,p-k
    \end{array}
    \hspace{-12pt}
    % \underset{\Lb\succeq \zerov, \rank(\Lb)=p-k }{\mathsf{argmin}} \,\,
    \begin{array}[t]{r}
      \frac{\rho}{2} \normop{\Lb - \left[\begin{smallmatrix}
\Ib_r & -\Bb^l \\
-\Bb^{l \top} & \Diag(\Bb^{l \top} \onev_r) 
\end{smallmatrix}\right] + \frac{1}{\rho}\Yb^l }_{F}^2 \\
-\log \det{\!^\ast}(\Lb )
    \end{array}
\end{align*}
The closed-form solution to this is given as
\begin{equation}
\label{eq:L_update}
\Lb^{l+1} = \frac{1}{2\rho} \Ub^l \left(\Sigmab^l + (\Sigmab^{l^2}+ 4\rho\Ib)^{1/2}\right)\Ub^{l \top},
\end{equation}
with $\Sigmab^l$ being a diagonal matrix having the largest $p-k$ eigenvalues of
$ \rho\left[\begin{smallmatrix}
\Ib_r & -\Bb^l \\
-\Bb^{l \top} & \Diag(\Bb^{l \top} \onev_r) 
\end{smallmatrix}\right] - \Yb^l$  
with corresponding eigenvectors $\Ub^l$.

\subsection{$\Bb$ update step:}
\label{sec:Gauss_B}
The subproblem associated with the update step of $\Bb$ is   formulated as follows:
\begin{equation}
\label{eq:B_subprob}
\begin{split}    
    \Bb^{l+1} &= \hspace{-8pt}
    \begin{array}[t]{c}
    {\mathsf{argmin}}\\
    \scriptstyle \Bb\geq \zerov, \,\Bb\onev_k = \onev_r
    \end{array}
    \hspace{-8pt}f_\Bb(\Bb) \\
 f_\Bb(\Bb) &=  \hspace{-3pt}
 \begin{array}[t]{l}
    \frac{p+\nu}{n} \sum_i \log \left(1 + \frac{h_i + \tr(\Bb \Gb_i(\Ab^l))}{\nu}\right)+
     % \frac{1}{n} \tr(\Bb \Gb(\Ab)) 
     \\
      \frac{\rho}{2} \normop{\Lb^{l+1} - \left[\begin{smallmatrix}
\Ib_r & -\Bb \\
-\Bb^\top & \Diag(\Bb^\top \onev_r) 
\end{smallmatrix}\right] + \frac{1}{\rho}\Yb^l }_{F}^2. 
    \end{array}
    \end{split}
\end{equation}
This problem does not admit a closed-form solution. Thus, we first   simplify the problem using the  majorization-minimization (MM) technique \cite{sun_majorization-minimization_2017}. For this, we need to
find a (smooth) majorization function $f_\Bb^S(\Bb;\Bb^l)$ with the
following properties:
\begin{align}
 \begin{split}
    f_\Bb^S(\Bb;\Bb^l) &> f_\Bb(\Bb),\quad \forall \Bb\neq \Bb^l \\
    f_\Bb^S(\Bb^l;\Bb^l) &= f_\Bb(\Bb^l),
     \end{split}
\end{align}
where   $\Bb^l$  is a constant matrix.

\begin{lemma}
\label{lemma:1}
  % Let $\bv = \vecc(\Bb)$. Then, $f_\bv(\bv) = f_\Bb(\Bb)$ 
  
The function $f_\Bb(\Bb)$ in \eqref{eq:B_subprob} can be majorized as
\begin{align}
\label{eq:b_surr_tot}
    \begin{array}[t]{ll}
   f_\Bb^S(\Bb;\Bb^l) =& \hspace{-6pt}
    \tr(\Bb\Hb^l)+
      \rho \normop{\Bb}_F^2+\frac{\rho }{2} \onev_r^\top \Bb\Bb^\top \onev_r +C(\Bb^l)
   \end{array}
\end{align}
 where 
\begin{equation*}
\begin{split}
     \Hb^l &= \Pb^l+ \rho\left(\Mb^{l \top}_{rk} +\Mb^l_{kr}  -\diag(\Mb^l)_{r+1:p}\onev_r^\top\right),
    \nonumber\\
        \Pb^l &= \frac{p+\nu}{n}   \sum_i \frac{\Gb_i(\Ab^l) }{h_i + \tr\left(\Bb^l\Gb_i(\Ab^l)\right) + \nu}, \nonumber\\
    \Mb^l &= \left[\begin{smallmatrix}
\Mb^{l}_{rr} & \Mb^{l}_{rk} \\
\Mb^{l}_{kr} & \Mb^{l}_{kk}
\end{smallmatrix}\right] = \Lb^{l+1} + \frac{1}{\rho}\Yb^{l}.
    \end{split}
\end{equation*} 

\end{lemma}
\begin{proof}
    See Appendix \ref{App:2}.
\end{proof}

Using this lemma, the upperbounded problem can be expressed as
\begin{equation}
\label{eq:B_update_simplex}
\begin{split}  
\Bb^{l+1}
&= \underset{\Bb\geq \zerov, \,\, \Bb\onev_k = \onev_r}{\mathsf{argmin}}\,\, f_\Bb^S(\Bb;\Bb^l).
    \end{split}
\end{equation}
To solve this, we use the projected gradient descent (PGD) method \cite{boyd_convex_2004} with the the step-size $\mu$ as:
\begin{equation}
\label{eq:B_update}
\begin{split}  
    \Bb^{m+1}
    &= \Pc_{\Sc_\text{row}}\left(\Bb^m - \mu\nabla f_\Bb^S(\Bb^m;\Bb^l) \right), \\
    \nabla f_\Bb^S(\Bb^m;\Bb^l) &= \Hb^{l \top} + 2\rho \Bb^m + \rho \onev_r\onev_r^\top \Bb^m.
    \end{split}
\end{equation}
where $\mathcal{P}_{\mathcal{S}_\text{row}}$ is a matrix operator that projects each row of the input matrix onto the probability simplex $\{\xv \geq \zerov \mid \langle \xv, \onev \rangle = 1\}$. For vector inputs, the simplex projection operator $\mathcal{P}_{\mathcal{S}}$ gives the solution to:
\begin{align}
\label{eq:simplex_proj}
    \Pc_{\Sc}(\xv_0) = \argminn_{\xv\geq\zerov,\,\xv^\top\onev = 1} \frac{1}{2} \normop{\xv-\xv_0}^2
\end{align}

% This problem can be solved efficiently via water-filling algorithms \cite{palomar_practical_2005}. 
This problem admits a nearly closed-form solution as given in Appendix~\ref{App:3}.

\subsection{$\Ab$ update step:}
The subproblem for the update step of $\Ab$ is
\begin{equation}
\begin{array}{cl}
    \mathsf{min} & 
f_\Ab(\Ab)\\
\mathsf{s.~t.} &  \Ab\geq \zerov, \,\Ab^\top\onev_r  =\onev_k, \,\Ab\odot \Id(\Bb^{l+1}=0) = \zerov,
\end{array}
\end{equation}
where
\begin{equation}
\label{eq:A_subprob}
   \begin{split}      
    f_\Ab(\Ab) =&
     \,\,\frac{p+\nu}{n} \sum_{i=1}^n \log \left(1 + \frac{h_i + \tr(\Bb^{l+1} \Gb_i(\Ab))}{\nu}\right).\\
    \end{split}
\end{equation}
Here,  we again apply the MM  to solve the problem, due to difficulty in obtaining the closed-form solution. 
We use the following lemma to construct a majorization  for $f_\Ab(\Ab)$.

\begin{lemma}
\label{lemma:2}
 Let $\av_j$ denote the $j$-th -column of $\Ab$. Then,  for constant $\Ab^l$,   $f_\Ab(\Ab)$ in \eqref{eq:A_subprob} can be majorized via
 % $f_\Ab^S(\Ab;\Ab^l)$ as 
\begin{align}
\label{eq:a_surr_tot}
    \begin{array}[t]{ll}
   f_\Ab^S(\Ab;\Ab^l) =& \hspace{-5pt}\sum_{j=1}^k g_{\av_j}^S(\av_j;\Ab^l) + C(\Ab^l),
   \end{array}
\end{align}
 where 
 \begin{align}
     \begin{split}    g_{\av_j}^S(\av_j;\Ab^l)& =\,b_j^{l+1}\av_j^\top   \tilde{\Sb}^l\av_j -2\bv_j^{l+1 \top}\tilde{\Sb}^l  \av_j,\\
         \tilde{\Sb}^l &= \frac{p+\nu}{n}\sum_i \frac{\tilde{\Sb}_i}{h_i + \tr(\Bb^{l+1}\Gb_i(\Ab^l))+\nu},
     \end{split}
 \end{align}
$b_j^{l+1} = \langle \bv^{l+1}_j, \onev\rangle$, and $\bv^{l+1}_j$ denotes the $j$-column of $\Bb^{l+1}$.
\end{lemma}
\begin{proof}
    See Appendix \ref{App:4}.
\end{proof}
Next, the update step for $\Ab$  is obtained by solving the following equation:
\begin{equation}
 \begin{array}{ccl}
   \Ab^{l+1}  =& \hspace{-5pt }\mathsf{argmin} & 
 \hspace{-5pt } f_\Ab^S(\Ab;\Ab^l)\\
& \hspace{-5pt } \mathsf{s. t.} &  \hspace{-5pt }  \Ab\geq \zerov, \,\Ab^\top\onev_r  =\onev_k, \\
& &\hspace{-5pt } \Ab\odot \Id(\Bb^{l+1}=0) = \zerov
\end{array}
\end{equation}
This problem can iteratively be solved for each   $\av_j$  via the PGD as   
\begin{align}
\label{eq:A_update}
{\av_j^{m+1}}  & = \begin{array}[t]{c}
    {\mathsf{argmin}}\\
    \scriptstyle {\av_j\geq \,\zerov, \, \,
    % \av_j\in S_j^{l+1}
    \av_j^\top \onev = 1
    } \\
     \scriptstyle{
     \av_j\odot\Id(\bv_j^{l+1}=0) = \,\zerov
      }
    % \Omega_{\av_j}
    \end{array} 
    % \,  \underset{\av \geq \zerov,\, \av\in S_j }{\mathsf{argmin}}\, 
    g_{\av_j}^S(\av_j,\Ab^{l})  \\[-2pt]
& =\Pc_{\Sc}\left( \av_j^m -2\eta\,\tilde{\Sb}^l(
          b_j^{l+1}
          \av_j^m
         - \bv_j^{l+1})\right) \odot 
         \Id(\bv_j^{l+1}>0)
         % S_j^{l+1}
         \nonumber 
\end{align}
  where $\eta$ is the step-size and
  $\Pc_{\Sc}$ refers to the  projection operator that maps a vector onto the simplex $\{\xv\geq \zerov,\, \xv^\top\onev = 1 \}$ (given in Appendix \ref{App:3}).

\subsection{Dual variable update step:}
Finally we have the update step for the dual 
% variables
variable
as
\begin{align}
\label{eq:dual_update}
\begin{split}  
\Yb^{l+1}   &= \Yb^{l} +  \rho \left( \Lb^{l+1} - \left[ \begin{smallmatrix}
\Ib_r & -\Bb^{l+1} \\
-\Bb^{l+1 \top} & \Diag(\Bb^{l+1 \top} \onev_r)
\end{smallmatrix}\right] \right). \\
\end{split}
\end{align}

\begin{algorithm}[t]
\caption{Proposed algorithm 
% (Stage 1) 
for  bipartite $k$-component graph learning}
\vspace*{1ex}
\begin{algorithmic}[1]
\State \textbf{Input:}   \textcolor{black}{$\tilde{\Xb}\in \Real^{r\times n}$} \quad \textbf{Parameters:}   $k$, 
% $\lambda$, 
$\nu$, $\rho$, $\mu$, and $\eta$ 
\State \textbf{Output:} $\Bb^{l}$
\State \textbf{Initialization:} \textcolor{black}{$\Ab^{0}$, $\Bb^0 = \Pc_{\Omega_\Bb}[(\Xb^\top\Xb/n)^\dagger]_{rk}$, $l=0$}
\Repeat
\State   Update $\Lb^{l+1}$ using  \eqref{eq:L_update}.
\State  Update  $\Bb^{l+1}$  by iterating   \eqref{eq:B_update} (starting from $\Bb^l$).
\State   Update $\Ab^{l+1}$ by iterating   \eqref{eq:A_update} (starting from $\Ab^l$).
\State Update the dual variable  using \eqref{eq:dual_update}.

\State Set $l \leftarrow l+1$.
\Until {a stopping criterion is satisfied}
\end{algorithmic}
\label{algorithm:alg1}
\end{algorithm}

\begin{figure*}[t]
\centering 
\hspace{-40pt}
\begin{subfigure}[t]{0.33\textwidth}
   \includegraphics[trim=0 80 0 60,clip,width=\textwidth]
   {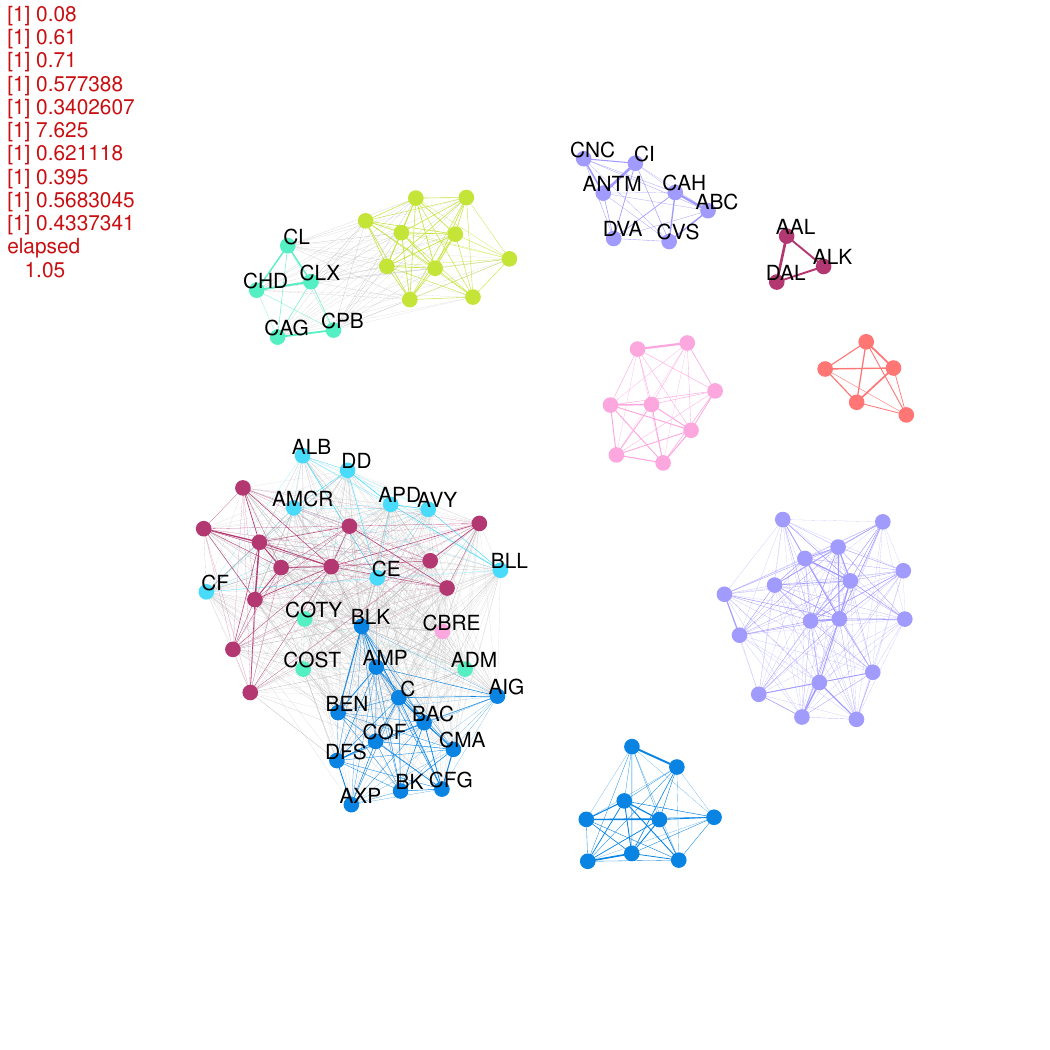}
     \caption{CLR  method \cite{nie_constrained_2016}
     }
\end{subfigure} 
\hspace{-20pt}
\begin{subfigure}[t]{0.33\textwidth}
   \includegraphics[trim=0 80 0 60,clip,width=\textwidth]
   {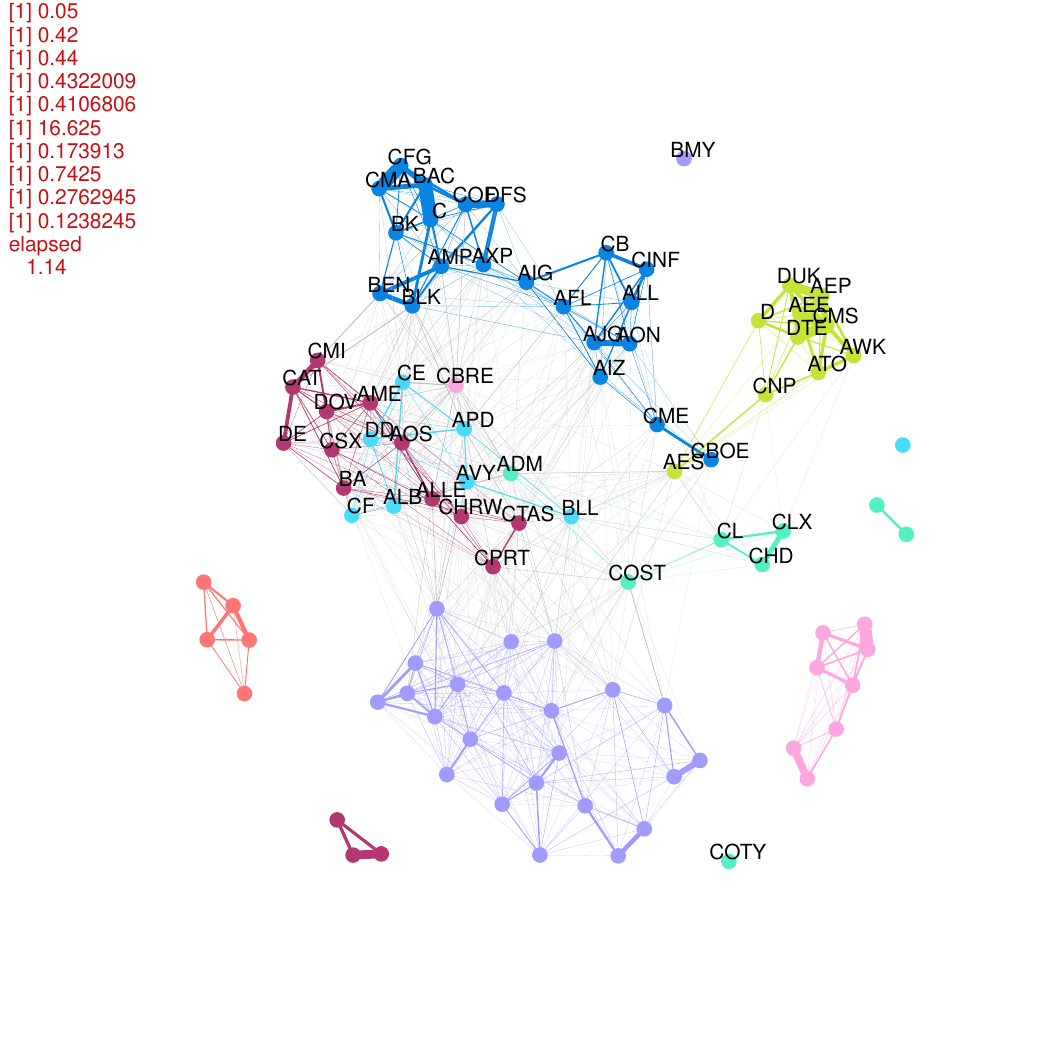}
     \caption{SGLA method    \cite{kumar_unified_2020}
     }
\end{subfigure} 
\hspace{-30pt}
\begin{subfigure}[t]{0.33\textwidth}
   \includegraphics[trim=0 80 0 60,clip,width=\textwidth]
   {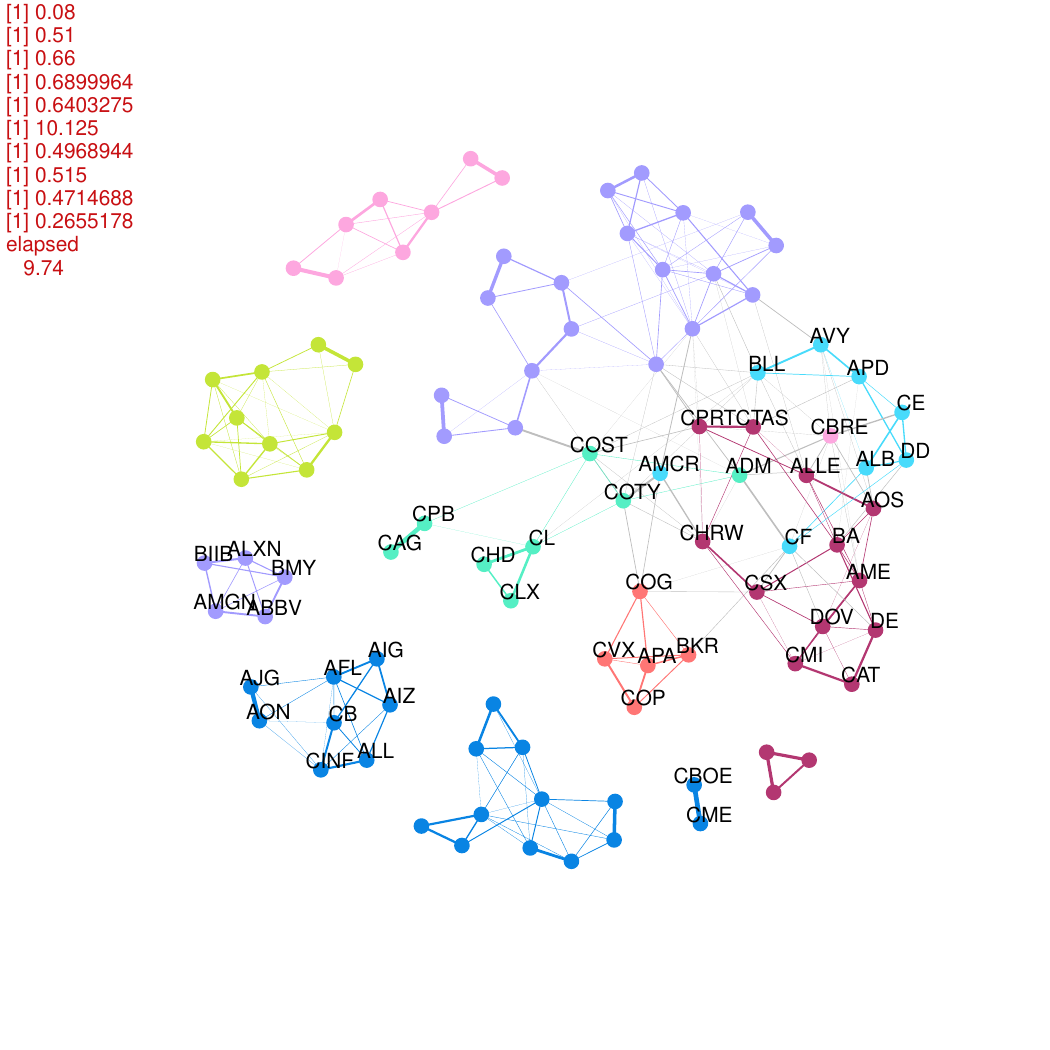}
     \caption{Fingraph method 
\cite{de_miranda_cardoso_graphical_2021}
     }
\end{subfigure} \\
\begin{subfigure}[t]{0.33\textwidth}
   \includegraphics[trim=0 80 0 60,clip,width=\textwidth]
   {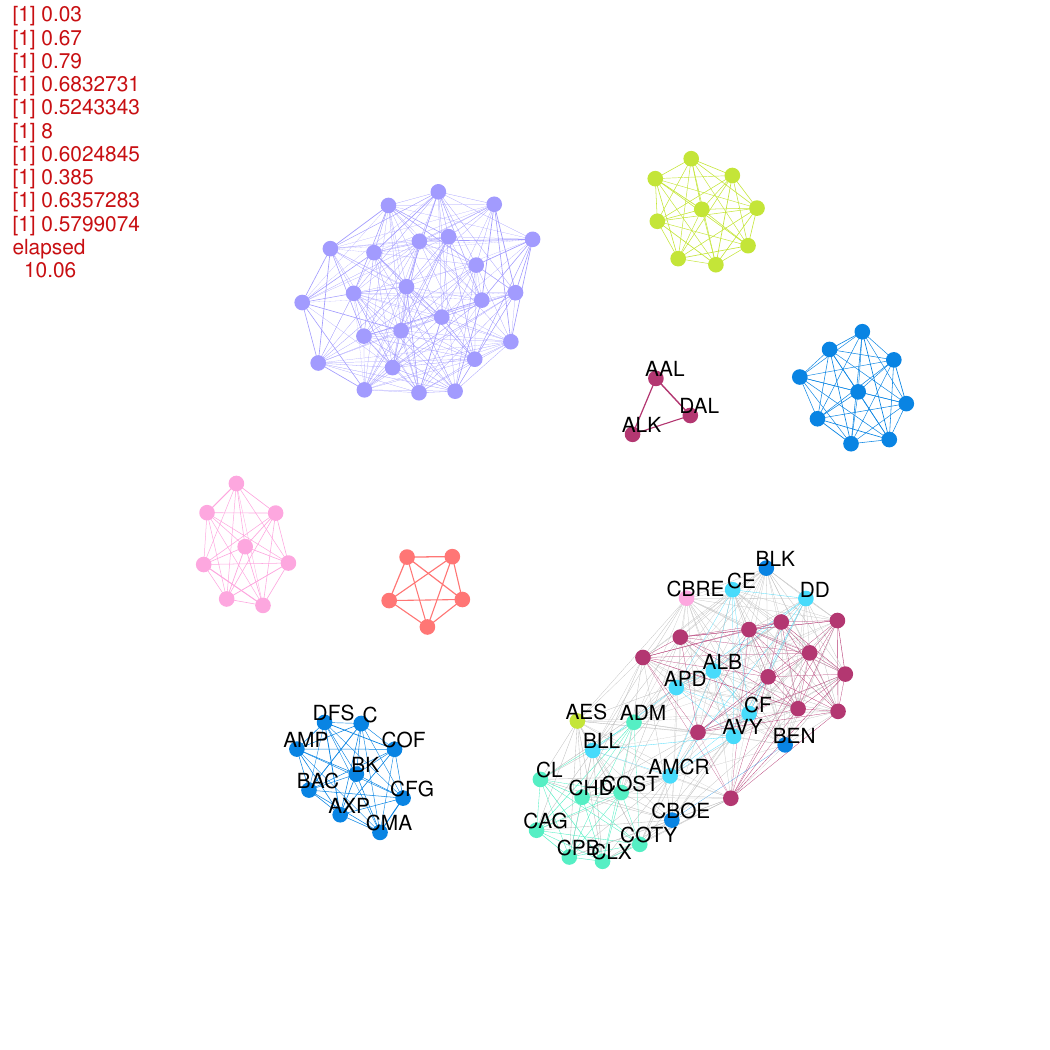}
     \caption{Javaheri et al. \cite{javaheri_graph_2023}  
     }
\end{subfigure} 
\hspace{-30pt}
\begin{subfigure}[t]{0.33\textwidth}
   \includegraphics[trim=0 80 0 60,clip,width=\textwidth]
   {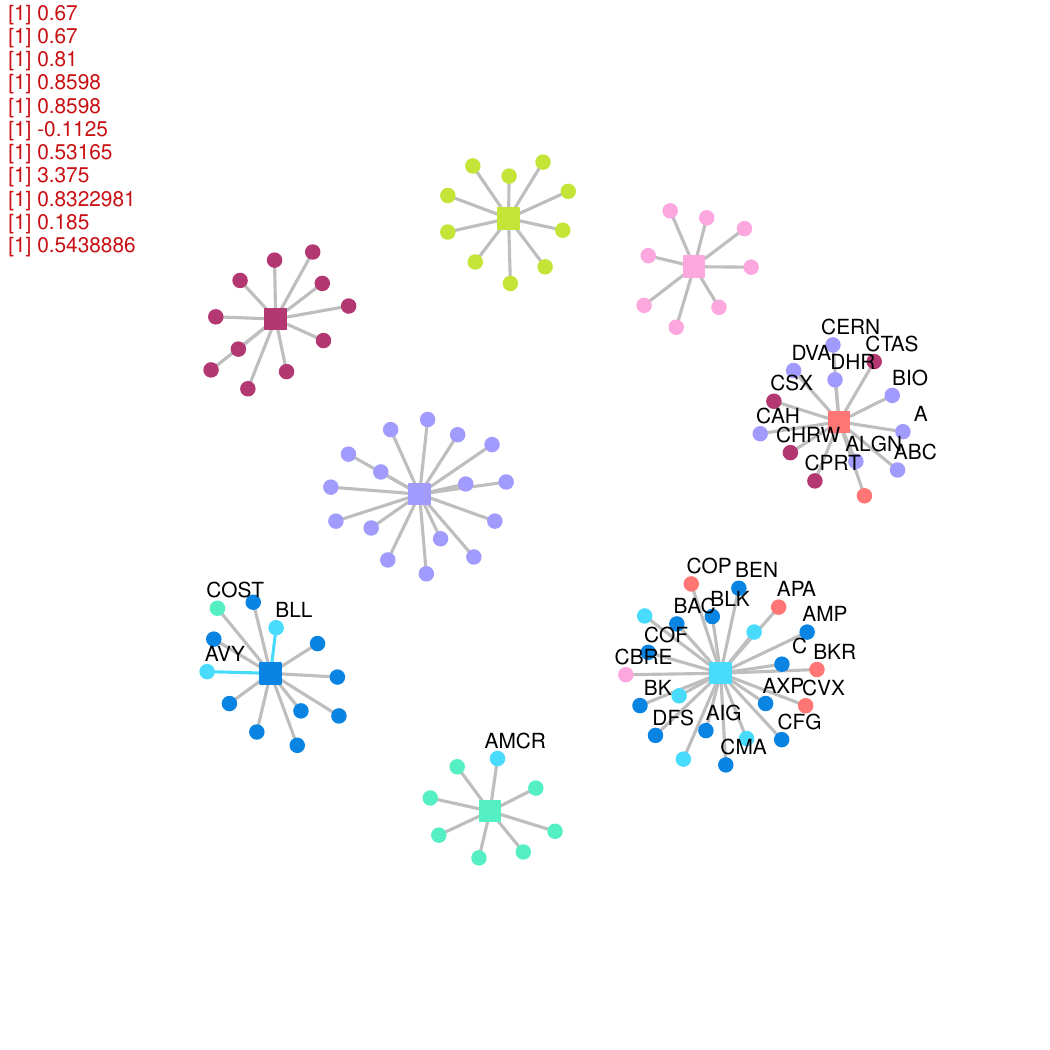}
     \caption{Proposed method with random uniform  $\Ab^0$.}
\end{subfigure} 
\hspace{-30pt}
\begin{subfigure}[t]{0.33\textwidth}
   \includegraphics[trim=0 80 0 60,clip,width=\textwidth]
   {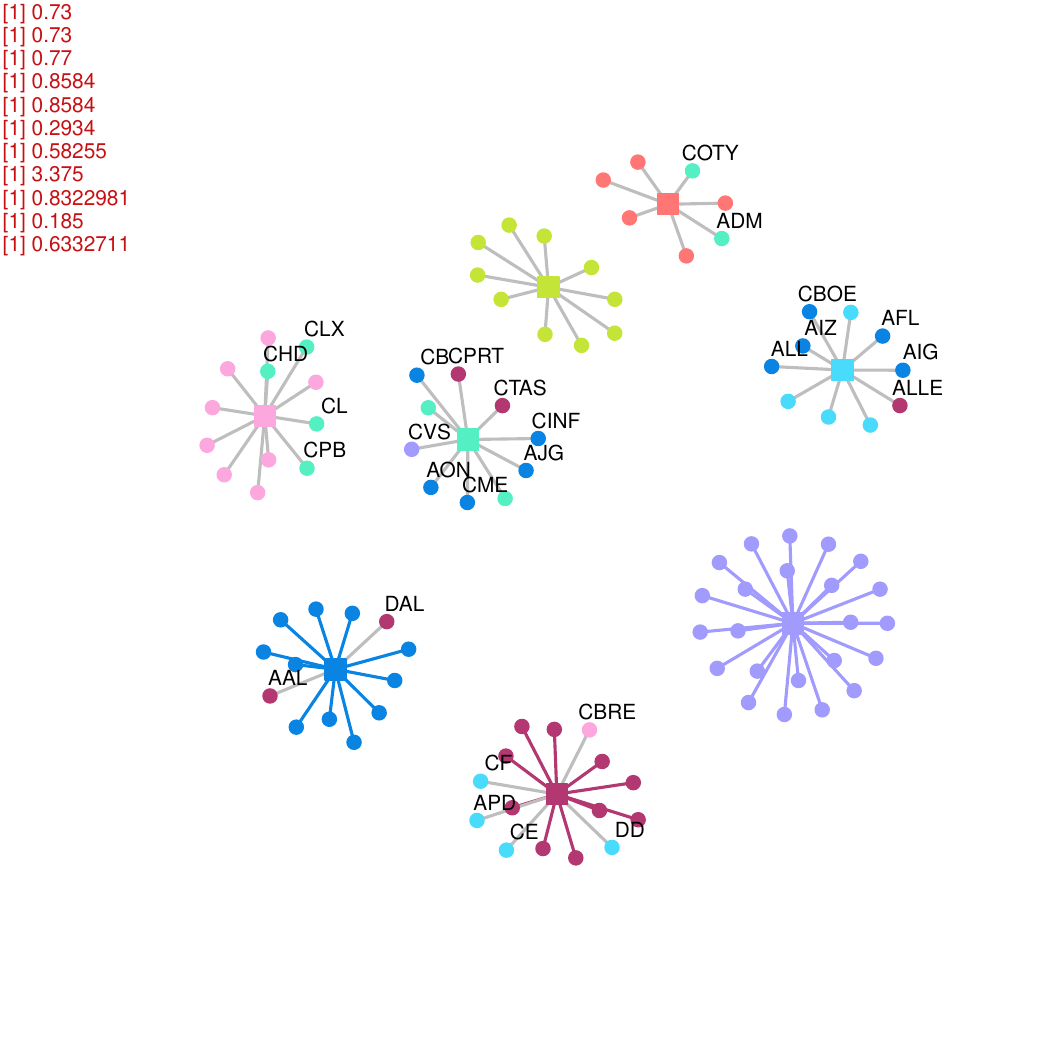}
     \caption{Proposed method with random normal  $\Ab^0$.}
\end{subfigure} 
\hspace{-10pt}
\begin{subfigure}[b]{0.08\textwidth}
 \includegraphics[trim=0 0 0 0,clip,width=\textwidth]{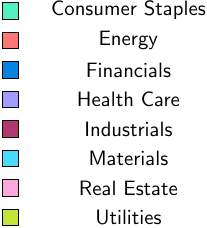}
\end{subfigure} 

\caption{The  graphs learned  from financial data corresponding to the log-returns of 100 stocks in S\&P500 index (including $k=8$ sectors).
}
\label{fig_fingraph_orig}
\end{figure*}

\section{Numerical Results}
\label{sec: Simulation}

In this part, we present numerical results to evaluate the performance of our proposed method for clustering heavy-tailed data. For this purpose, we utilize real-world financial data, specifically the log-returns of S\&P 500 stocks. Our experiment focuses on a subset of 100 stocks, divided into $k = 8$ clusters corresponding to financial sectors, with ground-truth cluster labels defined by the GICS classification standard\footnote{\href{https://www.msci.com/our-solutions/indexes/gics}{https://www.msci.com/our-solutions/indexes/gics}}. The log-returns of these stocks are calculated over a 1000-day period from January 2016 to January 2020. The resulting data matrix $\tilde{\Xb} \in \Real^{r \times n}$ consists of $r = 100$ rows (stocks) and $n = 1000$ columns (days).

To assess clustering performance, we employ accuracy (ACC), purity \cite{everitt_cluster_2011}, modularity (MOD) \cite{newman_modularity_2006}, adjusted Rand index (ARI) \cite{rand_objective_1971}, and Calinski-Harabasz index (CHI) \cite{calinski_dendrite_1974}. Accuracy and purity measure the ratio of true-positive 
% (correctly classified) 
labels to $p$. Accuracy considers the optimal alignment of inferred cluster labels to ground truth across all $k!$ permutations, while purity considers the majority label within each cluster as the ground truth. ARI, on the other hand, quantifies the similarity between the true and inferred cluster labels. Modularity also measures how disjoint the nodes with different labels are.
The Calinski-Harabasz index (CHI) is a reference-free criterion
that measures the ratio of the between-cluster separation to the within-cluster separation.

To run our method, we first obtain $\nu$ by
% In our method, the parameter $\nu$ is determined by 
fitting a multivariate Student-\textit{t} distribution to the data using the \texttt{fitHeavyTail} R package\footnote{\href{https://cran.r-project.org/package=fitHeavyTail}{https://CRAN.R-project.org/package/fitHeavyTail}}. 
We then consider two cases for initialization of $\Ab$. In the first case, $\Ab^0$ is sampled from a random uniform $U[0,1]$ distribution. In the other case, the entries of $\Ab^0$ are drawn from the normal distribution $\Nc(0,1)$. We later normalize $\Ab^0$ so that each column has unit sum. The augmented data matrix is then constructed as $\Xb = \left[\tilde{\Xb}^\top \quad \tilde{\Xb}^\top \Ab\right]^\top$. Using this, we compute the initial value of $\Bb^0 = \Pc_{\Omega_\Bb}[(\Xb^\top \Xb / n)^\dagger]_{rk}$, where $\Omega_\Bb$ denotes the set of feasible $\Bb$ matrices. Once the graph is learned, we assign the argument of the maximum entry of the $i$-th row of $\Bb^l$ as the cluster label for node $i$.

We compare our proposed method against several benchmark algorithms, including the constrained Laplacian rank (CLR) method \cite{nie_constrained_2016}, the SGLA method \cite{kumar_unified_2020}, the Fingraph method \cite{cardoso2021graphical}, and the method proposed by Javaheri et al. \cite{javaheri_graph_2023}.

Figure \ref{fig_fingraph_orig} illustrates the resulting graphs where the node colors represent the ground-truth clusters (sector indices).
 The clustering performance associated with these graphs are also given in Table \ref{table:1}. 
As shown in the table, the graph learned by the proposed method with random normal initialization  achieves  the highest   accuracy and the highest ARI, while the result with the uniform initialization gives the highest purity. 
Overall, the proposed method is shown to have superior  performance for financial data clustering.

\begin{table}[!]
\centering
 \caption{Clustering performance of the graphs shown in Fig. \ref{fig_fingraph_orig}}
 \label{table:1}
 \begin{tabular}{m{1cm}||c|c|c|c|c} 
     & ACC  & Purity & MOD & ARI & CHI   \\
     \hline
     
     \multicolumn{1}{c||}{CLR  \cite{nie_constrained_2016} }  & 
      0.61 & 
      0.71 & 
      0.34 &
      0.43 &
      4.73 \\
     \hline
     \multicolumn{1}{c||}{SGLA  \cite{kumar_unified_2020} }  & 
      0.42 & 
      0.44 & 
      0.41 &
     0.12 &
     3.08 \\
     \hline
     \multicolumn{1}{c||}{Fingraph \cite{de_miranda_cardoso_graphical_2021}  }  & 
      0.51 & 
    0.66 & 
    0.64 &
    \textcolor{black}{  0.26} &
   3.48 \\
     \hline
     \multicolumn{1}{c||}{Javaheri et al.  \cite{javaheri_graph_2023}  }  & 
    0.67 & 
    0.79 & 
    0.52 &
    0.57 &
    4.34 \\
    \hline
     \multicolumn{1}{c||}{  \multirow{1}{*}Proposed    (uniform $\Ab^0$)  }  
     &  0.67 
     &  \bf 0.81
     & \bf 0.85
     &  0.53
     & \bf 5.05 \\ 
     \hline
     \multicolumn{1}{c||}{  \multirow{1}{*}Proposed  (normal $\Ab^0$)} 
     & \bf 0.73 
     &   0.77 
      & \bf 0.85
     &  \bf 0.63
     & \bf 5.05\\
     \hline 
 \end{tabular}
 \end{table}

\section{Conclusion}
In this paper, we addressed the problem of learning bipartite $k$-component graphs for clustering data with heavy-tailed distributions, which are common in financial markets. Unlike existing methods that rely on access to data for both cluster centers and members in a bipartite graph model, our proposed approach addresses  this limitation by jointly inferring the  connections and the center nodes of the graph. Numerical experiments highlight the efficiency of the proposed method in clustering heavy-tailed data, particularly data from financial markets.

\appendices

\section{Proof of Lemma \ref{lemma:1}}
\label{App:2}
For simplicity in presentation, let us omit the superscript of the variables. Let $\Mb = \Lb + \frac{1}{\rho}\Yb$, $\mv=\diag(\Mb)$, $\mv_r = \mv_{1:r}$, and $\mv_k = \mv_{r+1:p}$. Then, we may write
\begin{align}
&\Bigl<\Mb,\left[\begin{smallmatrix}
    \Ib_r & -\Bb \\
    -\Bb^\top & \Diag(\Bb^\top \onev_r) 
   \end{smallmatrix}\right] \Bigr> \nonumber \\
   &=\mv_r^\top\onev_r -\langle\Mb_{rk},\Bb\rangle - \langle\Mb_{kr}, \Bb^\top\rangle +\mv_k^\top \Bb^\top \onev_r \nonumber\\
   &=\mv_r^\top\onev_r -\tr\left(\left(\Mb_{rk}^\top +\Mb_{kr}  -\mv_k\onev_r^\top\right) \Bb \right),
\end{align}
and
\begin{align}
    \normop{\left[\begin{smallmatrix}
    \Ib_r & -\Bb \\
    -\Bb^\top & \Diag(\Bb^\top \onev_r) 
   \end{smallmatrix}\right] }_F^2 =
   r + 2\normop{\Bb}_F^2 +  \onev_r^\top \Bb\Bb^\top \onev_r.
\end{align}

Thus, $f_\Bb(\Bb)$ in problem \eqref{eq:B_subprob}, can be simplified as 
\begin{align*}
    f_\Bb(\Bb) =& 
     \tr(\Bb\Rb)+
      \rho \normop{\Bb}_F^2+\frac{\rho }{2} \onev_r^\top \Bb\Bb^\top \onev_r \\
    & + \frac{p+\nu}{n} \sum_i \log \left(1 + \frac{h_i + \tr(\Bb \Gb_i(\Ab))}{\nu}\right),
\end{align*}
where $\Rb = \rho\left(\Mb_{rk}^\top +\Mb_{kr}  -\mv_k\onev_r^\top\right)$.

Now, using the logarithmic inequality, $\log(x)\leq x-1,\,\, \forall x>0$,
% \ref{App:5}, 
one can find an upperbound  as follows:
\begin{align*}
    \frac{p+\nu}{n} \sum_i \log \left(1 + \frac{h_i + \tr(\Bb \Gb_i(\Ab))}{\nu}\right) \leq \tr\left(\Bb\Pb_0\right) + C(\Bb_0),
\end{align*}
with
\begin{align*}
        \Pb_0 &= \frac{p+\nu}{n}   \sum_i \frac{\Gb_i(\Ab) }{h_i + \tr\left(\Bb_0\Gb_i(\Ab)\right) + \nu}. 
\end{align*}
and $C(\Bb_0)$ being a constant term.
Thus, we may propose a majorization function for $f_\Bb(\Bb)$ as
\begin{align*}
    f^S_\Bb(\Bb;\Bb_0) =& 
     \tr(\Bb\Hb)+
      \rho \normop{\Bb}_F^2+\frac{\rho }{2} \onev_r^\top \Bb\Bb^\top \onev_r +C(\Bb_0),
\end{align*}
where $\Hb_0 = \Pb_0 + \Rb$.
Thus, one can obtain \eqref{eq:b_surr_tot} by choosing $\Bb_0 =\Bb^l$.

\section{Simplex Projection Operator}
\label{App:3}
The Lagrangian function corresponding to  problem \eqref{eq:simplex_proj}  is as follows:
\begin{align}
    L(\xv, \alpha,\muv) = \frac{1}{2}\normop{\xv -\xv^0}^2  +\alpha(1 - \langle \xv, \onev\rangle) - \langle\muv,\xv\rangle
\end{align}
Based on the KKT conditions \cite{boyd_convex_2004}, the optimal solution satisfies:
\begin{align}
\begin{split}
    \frac{\partial}{\partial \xv} L(\xv, \alpha,\muv)  = \xv -\xv^0-\alpha\onev - \muv &= \zerov \\
    \langle \xv, \onev\rangle  &= 1\\
    \mu_i \phi_i &= 0  \\
    \phi_i &\geq 0  \\
    \mu_i &\geq 0
    \end{split}
\end{align}
A nearly closed-form solution to this can be obtained via a quick sorting method as follows
\begin{align}
\begin{split}
   \Pc_{\Sc}(\xv) &= \max\{0,\xv + \alpha \onev \} \\
        \alpha &= \max \left\{0,\frac{1-\sum_{i\in \Ic_\alpha}x_i }{\vert \Ic_\alpha\vert}\right\} \\
    \Ic_{\alpha} &= \left\{ i\vert -\alpha\leq  x_i \right\}
    \end{split}
\end{align}

\section{Proof of Lemma \ref{lemma:2}}
\label{App:4}
For simplicity we remove the superscripts in the notations.
Let $g_i(\Ab) = h_i + \tr\left(\Bb\Gb_i(\Ab)\right)$. We then have: 
\begin{align}
    g_i(\Ab) 
    &= h_i + \tr\left(\Bb   \left(-2 \Ab^\top \tilde{\Sb}_i + \diag \left(\Ab^\top \tilde{\Sb}_i \Ab \right)\onev_r^\top\right) \right)\nonumber\\
    &=h_i -2\langle\Ab,\tilde{\Sb}_i\Bb \rangle + \onev_r^\top\Bb \,\diag \left(\Ab^\top \tilde{\Sb}_i \Ab \right).
\end{align}
This can be decomposed in terms of the columns $\av_j$ of $\Ab$ as
\begin{align}
        g_i(\Ab) &=h_i -2\sum_{j=1}^k \langle \av_j, \tilde{\Sb}_i\bv_j\rangle + \sum_{j=1}^k b_j \av_j^\top\tilde{\Sb}_i\av_j,
\end{align}
where $b_j = \onev_r^\top\bv_j$.
Using the logarithmic inequality $\log (x) \leq x-1, \quad \forall x>0$,
% \ref{App:5},
we get
\begin{align*}
    \frac{p+\nu}{n}\log \left(1 + \frac{g_i(\Ab)}{\nu}\right)\leq  \frac{p+\nu}{n} \frac{g_i(\Ab)}{g_i(\Ab_0)+\nu} + C(\Ab_0).
\end{align*}
where $\Ab_0$ and $C(\Ab_0)$ are constant.
Hence, $f_\Ab(\Ab)$ in \eqref{eq:A_subprob}
can be majorized as
\begin{align*}
    g_\Ab(\Ab;\Ab_0)= & \,\,\frac{p+\nu}{n}\sum_i \frac{g_i(\Ab)}{g_i(\Ab_0)+\nu} =
    \sum_{j=1}^k g_{\av_j}(\av_j;\Ab_0), 
     \nonumber
\end{align*}
where
\begin{align}
         g_{\av_j}(\av_j;\Ab_0) =& \,\,b_j\av_j^\top   \tilde{\Sb}_0 \av_j -2\, \bv_j^\top \tilde{\Sb}_0\av_j+c_j(\Ab_0),
\end{align}
and
\begin{align}
    \tilde{\Sb}_0 &= \frac{p+\nu}{n}\sum_i \frac{\tilde{\Sb}_i}{h_i + \tr(\Bb\Gb_i(\Ab_0))+\nu}.
\end{align}
Choosing $\Ab_0=\Ab^l$, one yields \eqref{eq:a_surr_tot}.

%%%%%%%%%%%%%%%%%%%%%%%%%%%%%%%%%
%\begin{thebibliography}{1}
\bibliography{My_Library}
\bibliographystyle{IEEEbib}

%\end{thebibliography}

\end{document}